\documentclass[11pt]{article}
\usepackage[utf8]{inputenc}
\usepackage[a4paper, margin=1in]{geometry}
\usepackage{amsthm}
\usepackage{mathpazo}
\usepackage{graphicx} 
\usepackage{array} 

\usepackage{amsmath, amssymb, amsfonts, verbatim}
\usepackage{hyphenat,epsfig,subfigure,multirow}
\usepackage{nicefrac}
\renewcommand{\tilde}{\widetilde}
\renewcommand{\hat}{\widehat}

\usepackage{enumitem} 
\setlist[enumerate]{wide = 0pt, leftmargin=*}

\usepackage[usenames,dvipsnames]{xcolor}
\usepackage[ruled]{algorithm2e}

\usepackage{tcolorbox}
\tcbuselibrary{skins,breakable}
\tcbset{enhanced jigsaw}

\usepackage[normalem]{ulem}
\usepackage[compact]{titlesec}

\definecolor{DarkRed}{rgb}{0.5,0.1,0.1}
\definecolor{RURed}{rgb}{0.8,0.1,0.1}
\definecolor{DarkBlue}{rgb}{0.1,0.1,0.5}

\usepackage{nameref}
\definecolor{ForestGreen}{rgb}{0.1333,0.5451,0.1333}
\definecolor{Red}{rgb}{0.9,0,0}
\usepackage[linktocpage=true,
	pagebackref=true,colorlinks,
	linkcolor=DarkRed,citecolor=ForestGreen,
	bookmarks,bookmarksopen,bookmarksnumbered]
	{hyperref}
\usepackage[noabbrev,nameinlink]{cleveref}
\crefname{property}{property}{Property}
\creflabelformat{property}{(#1)#2#3}
\crefname{equation}{eq}{Eq}
\creflabelformat{equation}{(#1)#2#3}

\usepackage{bm}
\usepackage{url}
\usepackage{xspace}
\usepackage[mathscr]{euscript}
\usepackage{setspace}
\usepackage{mdframed}
\usepackage{tikz}
\usetikzlibrary{arrows}
\usetikzlibrary{arrows.meta}
\usetikzlibrary{shapes}
\usetikzlibrary{backgrounds}
\usetikzlibrary{positioning}
\usetikzlibrary{decorations.markings}
\usetikzlibrary{patterns}
\usetikzlibrary{calc}
\usetikzlibrary{fit}

\newtheorem{theorem}{Theorem}[section]
\newtheorem{lemma}{Lemma}[section]

\newtheorem*{claim*}{Claim}
\newtheorem*{proposition*}{Proposition}
\newtheorem*{lemma*}{Lemma}
\newtheorem*{corollary*}{Corollary}
\newtheorem*{remark*}{Remark}

\theoremstyle{definition}
\newtheorem{definition}{Definition}

\newtheorem*{problem*}{Problem}

\newtheorem{mdalg}{Algorithm}

\newtheorem{mdresult}{Result}

\newenvironment{tbox}{\begin{tcolorbox}[
		enlarge top by=5pt,
		enlarge bottom by=5pt,
		breakable,
		boxsep=0pt,
		left=4pt,
		right=4pt,
		top=10pt,
		arc=0pt,
		boxrule=1pt,toprule=1pt,
		colback=white
		]
	}
	{\end{tcolorbox}}

\newcommand{\dist}{\mathsf{dist}}

\newcommand{\cost}{\ensuremath{\textnormal{\textsf{cost}}}}

\newcommand{\Depth}{\mathsf{D}^\downarrow}

\title{Impossibility of Depth Reduction in Explainable Clustering}
 \author{Chengyuan Deng\\ Rutgers University\\\texttt{cd751@rutgers.edu} \and Surya Teja Gavva\\ Rutgers University\\\texttt{suryateja@math.rutgers.edu} \and Karthik C.\ S.\footnote{This work was supported  by the National Science Foundation under Grant CCF-2313372 and by a grant from the Simons Foundation, Grant
 Number 825876, Awardee Thu D. Nguyen.}\\ Rutgers University\\\texttt{karthik.cs@rutgers.edu} \and Parth Patel\\ Rutgers University\\\texttt{pp826@scarletmail.rutgers.edu} \and Adarsh Srinivasan\\ Rutgers University\\\texttt{adarsh.srinivasan@rutgers.edu}}

\date{}

\begin{document}

\maketitle

\begin{abstract}
	Over the last few years Explainable Clustering has gathered a lot of attention. Dasgupta et al.\ [ICML'20] initiated the study of explainable $k$-means and $k$-median clustering problems where the explanation is captured by a threshold decision tree which partitions the space at each node using axis parallel hyperplanes. 
	Recently, Laber et al.\ [Pattern Recognition'23] made a case to consider the depth of the decision tree as an additional complexity measure of interest. \vspace{0.1cm}

	In this work, we prove that even when the input points are in the Euclidean plane, then any depth reduction in the explanation incurs unbounded loss in the $k$-means and $k$-median cost. 
	Formally, we show that there exists a data set $X\subseteq \mathbb{R}^2$, for which there is a decision tree of depth $k-1$ whose $k$-means/$k$-median cost matches the optimal clustering cost of $X$, but every decision tree of depth less than $k-1$ has unbounded cost w.r.t. the optimal cost of clustering. We extend our results to the $k$-center objective as well, albeit with  weaker guarantees.

\end{abstract}

\clearpage

\section{Introduction}
\label{sec:intro}
Clustering is a fundamental task that needs to be carried out for numerous applications in the areas of data analysis and unsupervised learning. The input to clustering in these applications is typically 
a set of points in Euclidean space and the goal is to find a solution that minimizes a specified clustering objective. Arguably, the three most popular clustering objectives are  $k$-means, $k$-median, and $k$-center, and there is a long line
of theoretical research on designing algorithms for optimizing these objectives \cite{Vazirani}.

\begin{sloppypar}Explainable AI  has gained significant popularity in recent years~\cite{molnar2020interpretable, murdoch2019definitions} as it provides reliable and helpful insights about the model and its predictions, and more importantly allows human users to comprehend and trust the results and output created by machine learning algorithms.  Dasgupta et al.\ \cite{moshkovitz2020explainable}  introduced a notion of \emph{Explainable Clustering} for the $k$-means and $k$-median\footnote{Dasgupta et al.\ \cite{moshkovitz2020explainable}   introduced explainable clustering for the $k$-median objective in the $\ell_1$-metric, as it is easier to analyze the $k$-median objective in that metric space. Please see Section~\ref{sec:metric} for a small discussion.} objectives, which produces a threshold decision tree with $k$ leaves as the clustering result. That is, all cuts are axis-parallel at some thresholds (i.e. $x_r = \theta$, where $r$ is a certain coordinate and $\theta$ is a threshold value). The performance of explainable clustering is then evaluated by the \emph{price of explainability}, defined to be the ratio of the clustering objective cost   given by the decision tree to the optimal cost of the clustering. The initial attempt by~\cite{moshkovitz2020explainable} gave an upper bound of $O(k^2)$ (resp.\ $O(k)$) and a lower bound of $\Omega(k)$ 
	(resp.\ $\Omega(\log k)$) on the price of explainability of $k$-means clustering (resp.\ $k$-median clustering). \end{sloppypar}

Laber and Murtinho \cite{laber2021price} expanded the study of explainable clustering to the $k$-center objective, giving an upper bound of $O(\sqrt{d} k^{1-1/d})$ and a lower bound of $\Omega\left(\sqrt{d}\cdot \frac{k\sqrt{\ln \ln k}}{\ln^{1.5}k}\right)$ on the price of explainability for $k$-center clustering. 
Over the last two years, a sequence of works has almost closed the above gap on price of explainability for all three clustering objectives~\cite{gamlath2021nearly,MakarychevS21,esfandiari2022almost,charikar2022near,makarychev2023random,gupta2023price}. Additionally, other aspects and variants of explainable clustering have been explored in recent literature, some examples being, bicriteria approximability of the price of explainability \cite{FMR20,makarychev2022explainable}, computational complexity of finding good explanations \cite{laber2022computational}, providing explanations via oblique decision trees \cite{gabidolla2022optimal}, constructing a polyhedron around each cluster as an explanation \cite{lawless2022cluster},  removing outliers from clusters
for better explainability \cite{bandyapadhyay2022find}, and so on. 

Piltaver et al.\ \cite{PILTAVER2016333}  studied how tree structure parameters (such as the number of leaves, branching factor, tree depth) affect tree interpretability, and the authors  concluded that based on empirical data (from a survey with 98 questions answered by 69 respondents),
that the most important parameter is question depth (i.e., the length of the path that must be traversed from the root to a leaf to answer a question in a classification tree).
Building on this study, Laber, Murtinho, and Oliveira~\cite{laber2021shallow}  argued that in practice, an explanation via decision trees is useful only if the decision tree is of low depth (a.k.a.\ shallow decision trees).
Indeed, intuitively explaining leaves that are
far from the root involves many tests, which makes it harder to
grasp the model’s logic. In other words, short explanations are easy to understand and thus trust! 

The authors of \cite{laber2021shallow}  proposed an empirically-effective algorithm that leverages the price of the tree depth in the objective to return a desired cut. However, there has been no theoretical study in this direction. We thus raise the following question:

\begin{center}
	\emph{Can we always obtain short explanations without  compromising too much on the clustering quality?}
\end{center}

In order to answer the above question, we first introduce a notion called \emph{price of depth reduction}. Given an input point-set $X$, the \emph{price of depth reduction} to $h$, denoted $\Depth(X,h)$, is  the ratio of the clustering objective cost of the clustering given by the best decision tree of depth $h$  to the cost of the clustering given by the optimal decision tree.
Note that the depth of any decision tree with $k$ leaves is at least $\log_2 k$ and at most $k-1$. Therefore, we always have $\Depth(X,k-1)=1$.

One may wonder, if every point-set can be explained by  decision trees of low depth (say $O(\log k)$) while only paying a small price of depth reduction. In their seminal paper, Dasgupta et al.\ \cite{moshkovitz2020explainable} provided a point-set   in high dimensions,   which admits a decision tree of depth $k-1$ whose $k$-means/$k$-median/$k$-center cost is equal to the cost of an optimal clustering, but  the cost of every decision tree of depth $k-2$ for the point-set is arbitrarily high relative to the optimum.  Formally, they showed:

\begin{theorem}[Impossibility of Shallow Explanations in High Dimensions\footnote{In Section 3 in \cite{moshkovitz2020explainable}, the authors consider the point-set  $\{\vec{0},\vec{e}_1,\ldots, \vec{e}_d\}\subset \mathbb{R}^d$, but they might have instead considered the point-set $\{\vec{e}_1,-\vec{e}_1,\ldots, \vec{e}_d,-\vec{e}_d\}$, to obtain smaller value of $k$ w.r.t., $d$, as reflected in the theorem statement.}; Section 3 in \cite{moshkovitz2020explainable}]\label{thm:introhigh}
	The following holds for $k$-means, $k$-median, and $k$-center clustering objectives. 
	For every $k,d\in \mathbb{N}$, such that $d\ge k/2$, there is a point-set $X\in\mathbb{R}^d$, such that $\Depth(X,k-2)$ is unbounded. Moreover, the price of explainability of $X$ is 1.
\end{theorem}

The above unbounded inapproximability result stands in stark contrast to the literature on classical clustering, (for $k$-means, $k$-median, and $k$-center) for which we have constant factor polynomial time approximation algorithms \cite{Gonzalez85,HochbaumS86,AhmadianNSW20,Cohen-AddadEMN22}. One may attribute the above bad dataset to the \emph{Curse of Dimensionality} \cite{koppen2000curse,verleysen2005curse,taylor2019dynamic}. 

That said, while we know  high dimensional clustering is hard to approximate for the classical clustering problems \cite{FederG88,CK19, CKL21, CKL22}, the $k$-means and $k$-median objectives inapproximability is alleviated in fixed dimensions as they admit a PTAS \cite{CAKM16,FRS16a,KoR07,Cohen-Addad18,abs-1812-08664}. Even in the explainable clustering literature, we have algorithms which have lower price of explainability when the dataset is in fixed dimenesions  \cite{charikar2022near,esfandiari2022almost}.
However, our next result is that the impossibility of shallow explanations continues to hold for $k$-means and $k$-median objectives even in the plane.

\begin{theorem}[Impossibility of Shallow Explanations for $k$-median and $k$-means in the Plane; Informal version of \Cref{thm:2d-impossible}]\label{thm:introkmean}
	The following holds for $k$-means and $k$-median clustering objectives. 
	For every $k\in \mathbb{N}$, there is a point-set $X\in\mathbb{R}^2$, such that $\Depth(X,k-2)$ is unbounded. Moreover, the price of explainability of $X$ is 1.
\end{theorem}

We remark that while the construction of the dataset in \Cref{thm:introhigh} is quite simple, the dataset construction in the above theorem is quite involved and very subtle. It is inspired from what can be termed as ``blocking examples'', that are seen in   discrete geometry (for example, in the context of the conjectured extremal examples for the Pach-Tardos conjecture on axis parallel rectangles \cite{PT00}).

Our results for $k$-center in the plane are not as strong as the ones above for $k$-means and $k$-median. Nevertheless we are able to prove a lower bound of 2 on the price of depth reduction in the plane for $k$-center. 

\begin{theorem}[Lower Bound on Price of Depth Reduction for $k$-center in the Plane; Informal version of \Cref{thm:center}]\label{thm:introcenter}
	For every $k\in \mathbb{N}$ and every $\varepsilon>0$, there is a point-set $X\in\mathbb{R}^2$, such that $\Depth(X,k-2)$ is at least $2-\varepsilon$ for the $k$-center objective. Moreover, the price of explainability of $X$ is 1.
\end{theorem}

The rest of the paper is organized as follows. In \Cref{sec:pre}, we provide formal definitions that are relevant to this paper.  Our main impossibility results for $k$-means and $k$-median (i.e., \Cref{thm:introkmean}) in the plane is proved in \Cref{sec:plane} and in \Cref{sec:center} we prove \Cref{thm:introcenter} (which is related to the $k$-center objective). Finally, in \Cref{sec:conclusion} we conclude with some discussion and a couple of open problems.

\section{Preliminaries}
\label{sec:pre}
Throughout the paper, we say a point set $X$ has dimension $d$ with size $n$, and denote by $\dist(x, y)$ the Euclidean distance between points $x$ and $y$ in the point set.
Given two sets of points $X$ and $C$ in Euclidean space, we define
the $k$-means cost of $X$ for $C$
to be $\cost_2(C):=\underset{x \in X}{\sum}\left(\underset{c \in C}{\min}\ \left(\dist(x,c)\right)^2\right)$, the $k$-median cost to be $\cost_1(C):=\underset{x \in X}{\sum}\left(\underset{c \in C}{\min}\ \dist(x,c)\right)$, and the $k$-center cost to be $\cost_{\infty}(C):=\underset{x \in X}{\max}\left(\underset{c \in C}{\min}\ \dist(x,c)\right)$.
Given a set of points $X$, the $k$-means/$k$-median/$k$-center
objective is the minimum over all point-sets $C$  of cardinality $k$ of the $k$-means/$k$-median/$k$-center
cost of $X$ for $C$. Given a cluster  $Y\subseteq X$, the contribution of $Y$ to the
$k$-means/$k$-median cost of $X$ is simply the $1$-means/$1$-median cost of $Y$.

Dasgupta et al.~\cite{moshkovitz2020explainable} introduced  explainable clustering through decision trees. In a threshold decision tree $\mathsf{T}$ of depth $h_{\mathsf{T}}$ and $k$ leaves, every non-leaf node is labelled  by the hyperplane $x_r = \theta$ for some $r\in [d]$ and $\theta\in\mathbb{R}$. A point $x$ is then classified by the node to the left if it's $r^{\text{th}}$ coordinate is less than $\theta$
and to the right otherwise.

They also introduced the notion of \emph{price of explainability}, which is formally defined in \Cref{def:price-exp}. Furthermore, we define \emph{price of depth reduction} in \Cref{def:price-depth} to evaluate the cost of obtaining shallow explanations.

\begin{definition}[Price of explainability]
	\label{def:price-exp}
	Let $\Lambda(k)$ be the optimal  cost of  $k$-means/$k$-median/$k$-center objective for a  point-set $X$, and let $\Tilde{\Lambda}(k)$ be the optimal  cost for the same clustering objective but provided through explanations from a decision tree, the price of explainability is defined to be $\nicefrac{\tilde{\Lambda}(k)}{\Lambda(k)}$.
\end{definition}

\begin{definition}[Price of depth reduction]
	\label{def:price-depth}
	Let $\tilde \Lambda(k)$ be the optimal  cost of  $k$-means/$k$-medinas/$k$-center objective for a  point-set $X$ provided through explanations from a decision tree (of arbitrary depth), and let $\hat{\Lambda}_h(k)$ be the optimal cost for the same clustering objective but provided through explanations from a decision tree of depth $h$, then the price of depth reduction is defined to be $\Depth(X,h):=\nicefrac{\hat{\Lambda}_h(k)}{\tilde\Lambda(k)}$.
\end{definition}


 
\section{Impossibility of Depth Reduction in the Plane for $k$-means and $k$-median}
\label{sec:plane}
In this section, we prove a strong lower bound on the price of depth reduction for $k$-means and $k$-median in two dimensions. We construct a point-set $X \subseteq \mathbb{R}^2$, for which the optimal clustering can be explained by a decision tree of depth $k-1$ whose price of explanation is 1. However, we show that, the low-dimensionality does not help reduce the price of depth reduction. Formally, we prove the following theorem.

\begin{theorem}
	\label{thm:2d-impossible}
	For every $k\geq 3$, and every polynomial function $p:\mathbb{N}\to\mathbb{N}$, there exists a point-set $X \subset \mathbb{R}^2$  (where  $|X|$ is determined implicitly by $k$ and $p$, and $X$ can be represented in $|X|^{O(1)}$ bits) such that the following properties hold for $k$-means/$k$-median objectives:
	\begin{description}
		\item[Correctness:] There exists a decision tree of depth $k-1$ such that the price of explainability is 1.
		\item[Price of depth reduction:] For every $h<k-1$, every decision tree of depth $h$ yields a clustering whose cost is at least $p(|X|)$ times the optimum. In other words, $\Depth(X, k-2) \geq p(|X|)$.
	\end{description}
\end{theorem}

The proofs of the above theorem for the $k$-means and $k$-median objectives are very similar. Briefly, the underlying idea is to construct a point-set whose clusters are positioned such that the axis-parallel (either $x$ or $y$-axis) interval between any two clusters except one pair are blocked by a new cluster. Since the partition of clusters must be parallel to either $x$ or $y$-axis, we have that any cut line going through that interval will separate the points in the block cluster, and assign at least one point incorrectly. Therefore, to produce a perfect clustering, the decision tree must follow an `onion-peeling' approach, that is to output one cluster at each partition of the remaining space, leading to a decision tree of depth $k-1$. To get a lower bound on the price of depth reduction, it is imperative that the clusters are geometrically unbalanced, and we choose the distances between the clusters and the number of points in each cluster carefully.\par

\subsection{Point-set Construction and Optimal Clustering}
\label{subsec:data-opt}
We now describe the point-set that will be used to prove the lower bound on the price of depth reduction. 
We describe the point-set using weights as it simplifies the analysis. It is trivial to get rid of the weights by simply replacing a point $p$ having weight $w$, with $w$ copies of the unweighted point $p$. 

Informally, the point-set is constructed by  successively adding sets of points $C_i$ for $1 \leq i \leq k$ to it. Later, we will prove that the optimum $k$-means/$k$-median clustering is given by the partitioning $C:=(C_1, C_2, \ldots, C_k)$. We parameterize the  construction of the point-set through an increasing sequence $\{d_i\}_{i=1}^k$ and a decreasing sequence $\{w_i\}_{i=1}^k$ of integers, that fix the positions and the weights of the points respectively. Later we instantiate these parameters for a specific setting. 
\begin{tbox}
	\textbf{Weighted point-set} $X(w,d)$: \\
	\underline{Parameters}: Two sequences of positive integers, $\{w_i\}_{i=1}^{k}$ and  $\{d_i\}_{i=1}^{k}$\\
	Initialize the set $C_1$ having two copies of the same point, i.e.,  $p_1^1=p_1^2=(0,0)$ each with weight $w_1$, and another cluster $C_2$ having two copies of the same point, i.e., $p_2^1=p_2^2=(d_2,0)$ with weight $w_2$.  For $i = 3,4,\ldots, k$, add the sets $C_i$ as follows.
	\begin{itemize}
		\item If $i$ is odd, add two points with the same weight $w_i$,  where $p_i^1 = \left(\left(p_{i-2}^{1}\right)_x-1, \left(p_{i-2}^{1}\right)_y+d_i\right), p_i^2 = \left(\left(p_{i-1}^2\right)_x+1,\left(p_{i-2}^{2}\right)_y+d_i\right)$ (where $(p_i^j)_x$ and $(p_i^j)_y$ denote the $x$ and $y$ coordinates of the point $p_i^j$ respectively).
		
		\item If $i$ is even, add two points with the same weight $w_i$,  where $p_i^1 = \left(\left(p_{i-2}^{1}\right)_x+d_i, \left(p_{i-2}^{1}\right)_y-1\right), p_i^2 = \left(\left(p_{i-2}^{2}\right)_x+d_i, \left(p_{i-1}^2\right)_y+1\right)$.
	\end{itemize}
\end{tbox}
In \Cref{fig-2d}, we have an illustration of $X(w,d)$. In addition, we have shown the partition of the space using axis parallel lines (coming from a decision tree of depth $k-1$), which recovers the clustering $C$. We remark that in \Cref{fig-2d}, $k$ is odd and $d_k$ is the distance between clusters $C_k$ and $C_{k-2}$, which is not shown in the figure, but is similar to the distance $d_5$ between $C_3$ and $C_5$.

\begin{figure}[!h]
	\resizebox{\columnwidth}{!}{\begin{tikzpicture}
			
			
			\definecolor{color1}{RGB}{245, 224, 157}
			\definecolor{color2}{RGB}{199, 239, 240}
			\definecolor{color3}{RGB}{208, 245, 215}
			\definecolor{color4}{RGB}{210, 212, 197}
			
			\filldraw[color1,opacity=0.2] (-7,17) rectangle (29, 23);
			\filldraw[color2,opacity=0.4] (21,-6) rectangle (29, 17);
			\filldraw[color3,opacity=0.4] (-7,8) rectangle (15, 14);
			\filldraw[color4,opacity=0.4] (8,-6) rectangle (15, 8);
			\filldraw[pink,opacity=0.4] (-7,2) rectangle (8, 8);
			\filldraw[green,opacity=0.3] (3.5,-6) rectangle (8, 2);
			\filldraw[cyan,opacity=0.3] (-7,-6) rectangle (3.5, 2);

			\draw[violet, thick, dashed] (2,0) ellipse (1.25 and 0.6);
			\filldraw[black] (2,0) circle (2pt) node[red, anchor=north west]{\Large $w_1$};
			
			\draw[violet, thick, dashed] (5,0) ellipse (1.25 and 0.6);
			\filldraw[black] (5,0) circle (2pt) node[red, anchor=north west]{\Large $w_2$};
			
			\draw[black,thick, <->](2, -1) -- (5, -1) node[midway, blue, anchor = north]{\Large $d_2$};
			
			\draw[violet, thick, dashed] (3.5,4) ellipse (3.5 and 0.7);
			\filldraw[black] (1.5,4) circle (2pt) node[red, anchor=north west]{\Large$w_3$};
			\filldraw[black] (5.5,4) circle (2pt) node[red, anchor=north west]{\Large$w_3$};
			
			\draw[violet, thick, dashed] (11,2) ellipse (1.3 and 3.25);
			\filldraw[black] (11,-0.5) circle (2pt) node[red, anchor=north west]{\Large$w_4$};
			\filldraw[black] (11,4.5) circle (2pt) node[red, anchor=north west]{\Large$w_4$};
			
			\draw[black,thick, <->](8,0) -- (8,4) node[midway, blue, anchor = west]{\Large$d_3$};
			
			\draw[black,thick, <->](5,-1.5) -- (11,-1.5) node[midway, blue, anchor = north]{\Large$d_4$};
			
			\draw[violet, thick, dashed] (6.25,11) ellipse (6.5 and 1);
			\filldraw[black] (1,11) circle (2pt) node[red, anchor=north west]{\Large$w_5$};
			\filldraw[black] (11.5,11) circle (2pt) node[red, anchor=north west]{\Large$w_5$};
			
			\draw[black,thick, <->](9,4) -- (9,11) node[midway, blue, anchor = west]{\Large$d_5$};
			
			\draw[brown,line width=4pt, dotted](15.6,14.3) -- (20.4,16.7);

			\draw[violet, thick, dashed] (25,5.25) ellipse (2.5 and 10);
			\filldraw[black] (25,-3.5) circle (2pt) node[red, anchor=north west]{\Large$w_{k-1}$};
			\filldraw[black] (25,14) circle (2pt) node[red, anchor=north west]{\Large$w_{k-1}$};
			
			
			\draw[violet, thick, dashed] (10.75,20) ellipse (17 and 1.5);
			\filldraw[black] (-4,20) circle (2pt) node[red, anchor=north west]{\Large$w_{k}$};
			\filldraw[black] (25.5,20) circle (2pt) node[red, anchor=north west]{\Large$w_{k}$};
			
			\draw[black,thick, <->](22,13.5) -- (22,20) node[midway, blue, anchor = west]{\Large$d_{k}$};

	\end{tikzpicture}}
	
	\caption{Illustration of $X(w,d)$ with the explanation from a decision tree of depth $k-1$. The dotted oval indicates the optimal clustering assignment, and each colored block is the corresponding subspace produced by axis-parallel cut. Each point is associated with a weight $w_i$, and a new cluster is added with distance of $d_i$ from previous cluster $C_{i-2}$. }
	\label{fig-2d}
\end{figure}
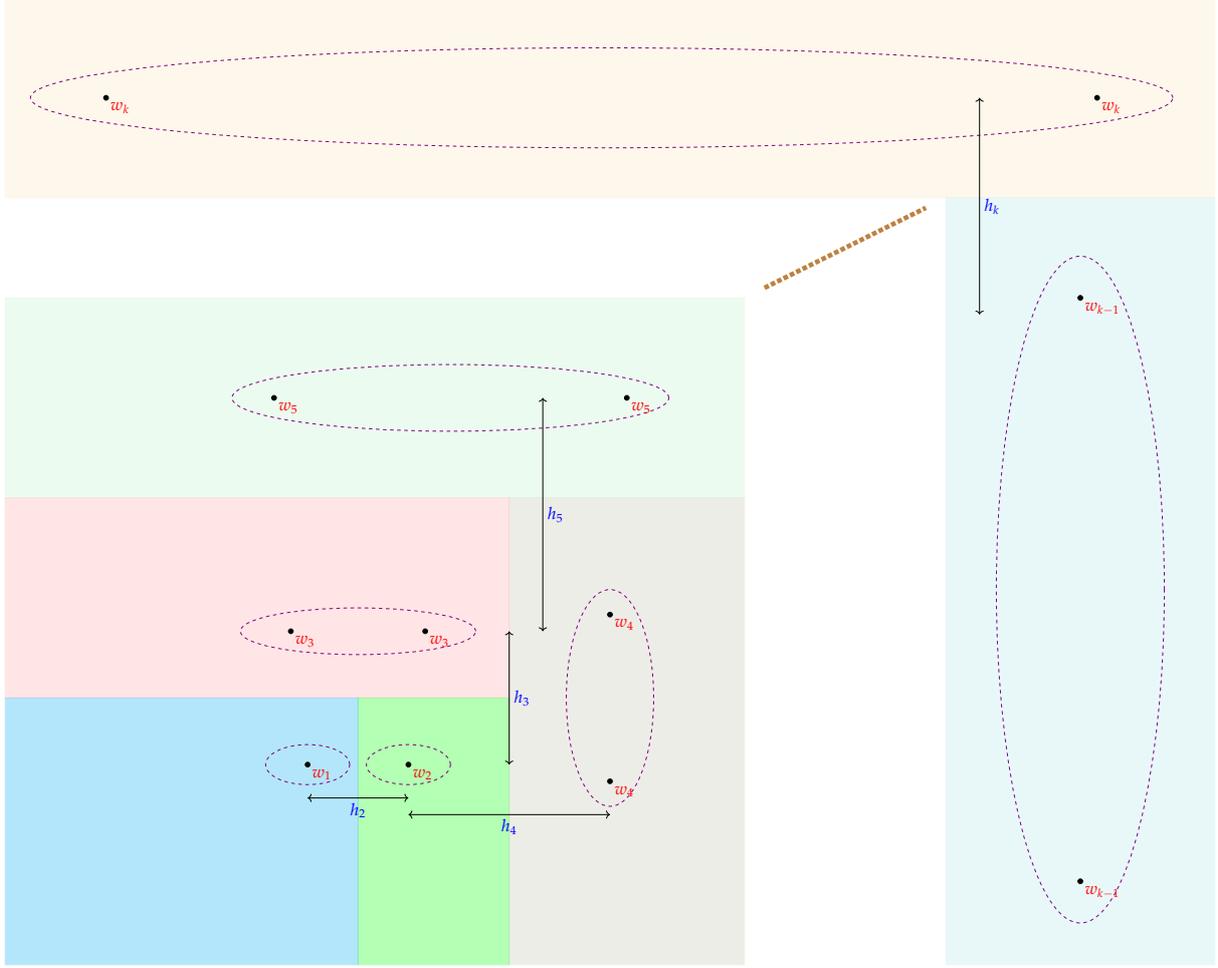

We make the following observations about the point-set $X(w,d)$:
\begin{enumerate}
	\item Every set $C_i$ has two points, both with the same weight and there are $2k$ points in total. Also note that for odd $i$, the $y$ coordinate are the same for the two points $p_i^1$ and $p_i^2$, and for even $i$, the $x$ coordinate are the same. 
	\item The straight line connecting the two points in the same cluster is either parallel to $x$-axis (if $i$ is odd) or $y$-axis (if $i$ is even).
	\item For any $r<i$ when $i,r$ are odd, we have $(p_i^1)_x < (p_r^1)_x$ and $(p_i^2)_x > (p_r^2)_x$. Similarly, when $r < i$ and $i,r$ are even, $(p_i^1)_y <(p_r^1)_y$ and $(p_i^2)_y > (p_r^2)_y$. Therefore, the set $C_i$ `blocks' any possible cut within the interval of the two points on the $y$ axis if $i$ is even or the $x$ axis if odd. This forbids any partition that gives a low-depth decision tree to classify the point-set into exactly $C_1, C_2, \dots, C_k$, which is the key property to complete the proof.
\end{enumerate}

Remarkably, our construction imposes an exponentially increasing interval between the latest added cluster $C_i$ and previous clusters. However, this does not result in a numerically egregious point set. As shown in \Cref{lem:size}, together with the specification of parameters $w, d$, $X(w,d)$ can be represented using polynomial space.

We give the complete proof of \Cref{thm:2d-impossible} in the next subsection.

\subsection{Proof of \Cref{thm:2d-impossible}}
We start with showing that the sets $C_i$ attain the optimum for both clustering objectives with proper choices of the sequences $\{w_i\}_{i=1}^k,\{d_i\}_{i=1}^k$. Further, any alternative clustering explanations can introduce a significant loss on the objective.
\begin{lemma}
	\label{lem:point-set-opt}
	We can choose the sequences $\{w_i\}_{i=1}^k,\{d_i\}_{i=1}^k$  such that the optimal clustering for $X(w,d)$ w.r.t. $k$-means/$k$-median objective is given by the set of clusters $C := (C_1,C_2, \ldots,C_{k})$ as defined. Furthermore, for any polynomial function p: $\mathbb{N} \rightarrow \mathbb{N}$, any other choice of explainable clustering $\Tilde{C} \neq C$ has $k$-means/$k$-median cost at least $p(|X|)$ times the optimum.
\end{lemma}

\begin{proof}
	Consider an alternative clustering $\tilde{C}:=(\tilde C_1, \ldots, \tilde C_k)$ of the weighted point-set $X$. We now specify the parameters that lead to $\cost(\Tilde{C})  \geq p(|X(w,d)|) \cdot \cost(C)$. 
	
	Let $d_i=\left( 4\sqrt{k \cdot p(2k)} \right)^i$ and $w_i= \nicefrac{d_k^2}{d_{i-1}^2}$ for $i \geq 3$ and $w_1=w_2=\nicefrac{d_k^2}{d_2^2}$. This choice of parameters implies that the $k$-means cost of each cluster $C_i$ are approximately equal. This is because the distance between the two points in the cluster $C_{2i+1}$ is $d_2+d_4+ \dots + d_{2i}+i+1$ and the distance between two points in the cluster $C_{2i}$ is $d_3+d_5+\dots+d_{2i-1}+i$. Hence, we have $d_i \geq \sum_{j \leq i}d_j$, $\frac{w_i d_{i-1}^2}{2}\leq \cost(C_i) \leq 4 w_i d_{i-1}^2$, for every $i \in [k]$. This implies that
	$$\forall{i\in[k]},\ \ \  \frac{d_k^2}{2} \leq \cost(C_i) \leq 4 d_k^2. $$
	As $\Tilde{C} \neq C$, there exist $i,j,j' \in [k]$, with $j > j'$ and $a,b \in \{1,2\} $ such that $p_j^a, p_{j'}^b \in \Tilde{C}_i$. Let $\Tilde{\mu}_i$ be the center of $\Tilde{C}_i$. As the distance between $p_j^a$ and $p_{j'}^b$ is at least $d_j$, this implies that either $\dist(p_{j}^a, \Tilde{\mu}_i) \geq d_j/2$ or $\dist(p_{j'}^b, \Tilde{\mu}_i) \geq d_j/2$. As the point $p_j^a$ has a smaller weight than the point $p_{j'}^b$, we can lower bound the cost contributed by the cluster $\Tilde{C}_i$ by
	\begin{align*} \cost(\Tilde{C}_i) \geq \frac{w_jd_j^2}{4}=\frac{w_jd_{j-1}^2}{4}\frac{d_j^2}{d_{j-1}^2} &\geq \frac{\cost(C)}{16k} \frac{d_j^2}{d_{j-1}^2} \geq p(2k) \cost(C) = p(|X|) \cost(C) \end{align*}
	
	The analysis of $k$-median is almost the same as $k$-means except that the sequence $\{d_i\}_{i=1}^k$ is defined by $d_i=\left(4 k^3 \cdot p(k)^3 \right)^i$, and the sequence $\{w_i\}_{i=1}^k$ is $w_i= \frac{d_k}{d_i}$.   
\end{proof}

With the specification of $d, w$ in \Cref{lem:point-set-opt}, we show the size of the point-set can be represented using bits 
polynomial in $|X|$.

\begin{lemma}[Size of this Point-set] \label{lem:size}
	The point-set $X(w,d)$ can be stored using $O(|X|^{O(1)})$ bits.
\end{lemma}
\begin{proof}
	The $x$-coordinates for every point in this point-set are integral and upper bounded by $k+\sum_{i=1}^{\lfloor k/2 \rfloor}d_{2i}$. Similarly, $y$-coordinates are upper bounded by $k+\sum_{i=1}^{\lfloor k/2 \rfloor}d_{2i+1}$. Recall that $\{d_i\}_{i=1}^k$ is an increasing sequence and $\{w_i\}_{i=1}^k$ is decreasing. Hence, each point in the point-set can be stored in space $O(\log d_k)$. Further, each weight can be stored in space $O( \log w_1)$ (the weights are scaled to ensure integrality). Hence, the entire point-set can be stored in space $O(k(\log d_k+\log w_1))$. As become evident, both $d_k$ and $w_1$ are $O(k^k)$. Recall $|X| = 2k$, hence $X(w,d)$ requires at most $O(|X|^{O(1)})$ bits to represent.
\end{proof}

Next, we show that for a weighted point-set $X(w,d)$ described above, a $k-1$-depth decision tree gives the optimal clustering, which concludes the correctness of \Cref{thm:2d-impossible}.

\begin{lemma}
	\label{lem:imp-correct}
	There exists a decision tree of depth $k-1$ which outputs the optimal clustering for the weighted point-set $X$ w.r.t. $k$-means/$k$-median objective.
\end{lemma}
\begin{proof}
	We describe a decision tree of depth $k-1$ to output the clustering $C=(C_1, \ldots, C_k)$. It has depth $k-1$, with there being a leaf at every level of the decision tree. At the $i^{\text{th}}$  level of the decision tree, if $k-i+1$ is even, the cut is along the $x$-axis with threshold at $x=(p_{k-i+1}^1)_x-\nicefrac{1}{2}$. For any point, if the $x$-coordinate of that point is larger than this threshold, it is classified to the cluster $C_{k-i+1}$. Otherwise, if $i \neq k-1$ it proceeds to the next level of the tree, and if $i=k-1$, it is classified to the cluster $C_1$. If $k-i+1$ is odd, the cut is along the $y$-axis with threshold $y=(p_{k-i+1}^1)_y-\nicefrac{1}{2}$. If the $y$-coordinate of a point is larger than this threshold, it is classified to the cluster $C_{k-i+1}$. Otherwise, if $i \neq k-1$ it proceeds to the next level of the tree, and if $i=k-1$, it is classified to the cluster $C_1$.
\end{proof}
We now show that this particular clustering cannot be induced by any decision tree of lower depth. And any shallower decision tree yields another clustering.

\begin{lemma}
	\label{lem:imp-misclassify}
	Every decision tree of depth less than $k-1$ outputs a clustering which misclassifies (w.r.t. the clustering $C$) at least one point.
\end{lemma}
\begin{proof}
	We say that a node $v$ in the decision tree is \emph{trivial} for the point set $X$, if the node has a child, which contains exactly the same points as $v$. In other words, a trivial node does not partition any points in the remaining point-set into a cluster. We can assume that the minimum depth decision tree does not contain any trivial nodes because we can contract that node without changing the final classification of $X$, and the depth can only decrease as a result of this contraction. Hence, without loss of generality, we can assume that every node in the tree is non-trivial. 
	
	We now apply induction on $k$. If there are only two points in the point-set, there is only one way to correctly classify them and that uses a decision tree of depth $1$. This proves the base case. 
	
	Let $k$ be odd. The two points in the cluster $C_k$: $p_k^1 = \left(\left(p_{k-2}^{1}\right)_x-1, \left(p_{k-2}^{1}\right)_y+d_i\right)$ and $p_k^2 = \left(\left(p_{k-1}^2\right)_x+1,\left(p_{k-2}^{2}\right)_y+d_k\right)$ are to the left and to the right of every point in $C_j$, for every $j < k$.  We first consider the case that the first threshold (starting from the root of the decision tree) is along the $x$-axis. Any decision tree that outputs the clustering $C = (C_1, \ldots, C_k)$ must classify $p_k^1$ and $p_k^2$ in the same cluster. The only possible cuts along the $x$-axis that satisfy this property are cuts with threshold larger than the $x$-coordinate of $p_k^2$ or cuts with a threshold lesser than the $x$-coordinate of $p_k^1$. In both these cases, the second level of the decision tree has two nodes, but all the points in the point set get classified into one of them, which would make it trivial. Hence, the first threshold cut must be along the $y$-axis. \par
	
	Now, we consider the case that the first threshold cut is along the $y$-axis. In order to output the clustering $C$, the cut would have to have $y$-coordinate less than $\left(p_{k-1}^1\right)_y$ or greater than $\left(p_{k-1}^2\right)_y$. Otherwise, the cut would classify $p_{k-1}^1$ and $p_{k-1}^2$ into different clusters. If the threshold has $y$-coordinate greater than $\left(p_k^1\right)_y$ or less than $\left(p_{k-1}^1\right)_y$, the cut would be trivial as the second level of the decision tree would have two nodes, but all the points in the point set would get classified into one of them, which would make it trivial. Hence, the threshold lies between $\left(p_k^1\right)_y$ and $\left(p_{k-1}^2\right)_y$. This implies that the  second level of the decision tree contains two nodes, and all the points in the set $C_k$ get classified into one of them. As the decision tree must output the clustering $\Tilde{C}$ on the given point set and is non trivial, that node must be a leaf. All the remaining points in the point set (which is another point set in the same family with parameter $k-1$) go to the other node in the second level, and the inductive hypothesis implies subtree rooted at this node must have depth at least $k-2$, which implies that the decision tree we are considering must have depth $k-1$.
	
	The proof proceeds similarly when $k$ is even, except that the only possible non trivial cuts in the first level would be along the $x$-axis with thresholds between $\left(p_{k}^1\right)_x$ and $\left(p_{k-1}^2\right)_x$.
\end{proof}

This completes the proof of Theorem \ref{thm:2d-impossible}. With the specification of $d,w$, \Cref{lem:size} shows that the point-set can be represented using $|X|^{O(1)}$ bits. The correctness is implied by \Cref{lem:imp-correct}, the price of depth reduction lower bound is shown together by \Cref{lem:point-set-opt} and \Cref{lem:imp-misclassify}.

\section{Lower Bound on Price of Depth Reduction in the Plane for $k$-center}
\label{sec:center}
In this section, we prove a lower bound on the price of depth reduction for the $k$-center objective. Specifically, we construct a point-set $X \subseteq \mathbb{R}^2$, for which the optimal clustering can be explained by a decision tree of depth $k-1$ whose price of explanation is 1. Furthermore, we show that the price of depth reduction is at least $2-\varepsilon$, for any small $\varepsilon>0$.  Formally, we state our result as follows:

\begin{theorem}
	\label{thm:center}
	For every $k\geq 3$ there exists a point-set $X \subset \mathbb{R}^2$  such that the following properties hold  for the $k$-center objective for all $\varepsilon>0$:
	\begin{description}
		\item[Correctness:] There exists a decision tree of depth $k-1$ such that the price of explainability is 1. 
		\item[Price of depth reduction:] For every $h<k-1$, every decision tree of depth $h$ yields a clustering whose cost is at least $(2-\varepsilon)$ times the optimum. In other words, $\Depth(X, k-2) \geq 2-\varepsilon$.
	\end{description}
\end{theorem}

\subsection{Point-Set construction}

The construction below is similar in many aspects to the  point-set constructed in \Cref{sec:plane}. One major difference is that each cluster has 4 points, and we use the notation of $p_i^a$ where $i \in [k], a\in[4]$ to indicate the $a^{\text{th}}$ point in the $i^{\text{th}}$ cluster. Also, the points are not weighted here. Again, as in Section~\ref{sec:plane}, we define the point-set through an underlying clustering $C:=(C_1,\ldots ,C_k)$. 
\begin{tbox}
	\textbf{Point-set $X$ for $k$-center:}
	\begin{itemize}
		\item Initialize the first two clusters $C_1, C_2$ symmetrically across the $y$-axis, each containing four points. In $C_1$, the four points are located at $p_1^1 = (-1,0), p_1^2 = (-1.5,0.5), p_1^3= (-1,1), p_1^4 = (-0.5,0.5)$. In $C_2$, the four points are located at $p_2^1 = (1,0), p_2^2 = (0.5,0.5), p_2^3= (1,1), p_2^4 = (1.5,0.5)$. 
		\item Construct $C_3$ with four points, where $(p_3^a)_x = (p_1^a)_x+1$ and $(p_3^a)_y = (p_1^a)_y+2$ for $a \in [4]$. 
		\item Construct $C_4$ with four points, where $(p_4^a)_x = (p_2^a)_x+2$ and $(p_4^a)_y = (p_2^a)_y+1$ for $a \in [4]$. 
		\item For $i = 5, \ldots, k$, add cluster $C_i$ containing four points iteratively, where $(p_i^a)_x = (p_{i-2}^a)_x+2$ and $(p_i^a)_y = (p_{i-2}^a)_y+2$ for $a \in [4]$.
	\end{itemize}
\end{tbox}

From the construction above, it is easy to see that the $k$-center cost of $X$ is at most $\nicefrac{1}{2}$ (as given by the clustering $C$). Further, it only takes $\tilde{O}(|X|)$ bits to store the whole point-set.

In \Cref{fig-kc}, we have an illustration of point-set $X$. In addition, we have shown the partition of the space using axis parallel lines (coming from a decision tree of depth $k-1$), which recovers the clustering $C$.

\begin{figure}[!h]
	\resizebox{\columnwidth}{!}{\begin{tikzpicture}
			
			\definecolor{color1}{RGB}{245, 224, 157}
			\definecolor{color2}{RGB}{199, 239, 240}
			\definecolor{color3}{RGB}{208, 245, 215}
			\definecolor{color4}{RGB}{210, 212, 197}

			\filldraw[color1,opacity=0.4] (22,-1) rectangle (27, 23);
			\filldraw[color2,opacity=0.4] (-1,18) rectangle (22, 23);
			\filldraw[color3,opacity=0.4] (15,-1) rectangle (19, 15);
			\filldraw[color4,opacity=0.4] (-1,11) rectangle (15, 15);
			\filldraw[cyan,opacity=0.4] (11,-1) rectangle (15, 11);
			\filldraw[pink,opacity=0.4] (-1,7) rectangle (11, 11);
			\filldraw[green,opacity=0.4] (7,-1) rectangle (11, 7);
			\filldraw[cyan!50!red,opacity=0.4] (-1,3) rectangle (7, 7);
			\filldraw[pink!50!orange,opacity=0.4] (3,-1) rectangle (7, 3);
			\filldraw[yellow,opacity=0.4] (-1,-1) rectangle (3, 3);

			\draw[violet, thick, dashed] (1,1) ellipse (1 and 1);
			\filldraw[black] (1,2) circle (3pt) node[red, anchor=north west]{\Large$p_1^3$};
			\filldraw[black] (1,0) circle (3pt) node[red, anchor=north west]{\Large$p_1^1$};
			\filldraw[black] (2,1) circle (3pt) node[red, anchor=north west]{\Large$p_1^4$};
			\filldraw[black] (0,1) circle (3pt) node[red, anchor=north west]{\Large$p_1^2$};
			
			\draw[violet, thick, dashed] (5,1) ellipse (1 and 1);
			\filldraw[black] (5,2) circle (3pt) node[red, anchor=north west]{\Large$p_2^3$};
			\filldraw[black] (5,0) circle (3pt) node[red, anchor=north west]{\Large$p_2^1$};
			\filldraw[black] (4,1) circle (3pt) node[red, anchor=north west]{\Large$p_2^2$};
			\filldraw[black] (6,1) circle (3pt) node[red, anchor=north west]{\Large$p_2^4$};
			
			\draw[violet, thick, dashed] (3,5) ellipse (1 and 1);
			\filldraw[black] (2,5) circle (3pt);
			\filldraw[black] (4,5) circle (3pt);
			\filldraw[black] (3,6) circle (3pt);
			\filldraw[black] (3,4) circle (3pt);
			
			\draw[violet, thick, dashed] (9,3) ellipse (1 and 1);
			\filldraw[black] (8,3) circle (3pt);
			\filldraw[black] (10,3) circle (3pt);
			\filldraw[black] (9,2) circle (3pt);
			\filldraw[black] (9,4) circle (3pt);
			
			\draw[violet, thick, dashed] (7,9) ellipse (1 and 1);
			\filldraw[black] (7,10) circle (3pt);
			\filldraw[black] (7,8) circle (3pt);
			\filldraw[black] (8,9) circle (3pt);
			\filldraw[black] (6,9) circle (3pt);
			
			\draw[violet, thick, dashed] (13,7) ellipse (1 and 1);
			\filldraw[black] (12,7) circle (3pt);
			\filldraw[black] (14,7) circle (3pt);
			\filldraw[black] (13,6) circle (3pt);
			\filldraw[black] (13,8) circle (3pt);
			
			\draw[violet, thick, dashed] (11,13) ellipse (1 and 1);
			\filldraw[black] (11,14) circle (3pt);
			\filldraw[black] (11,12) circle (3pt);
			\filldraw[black] (12,13) circle (3pt);
			\filldraw[black] (10,13) circle (3pt);
			
			\draw[violet, thick, dashed] (17,11) ellipse (1 and 1);
			\filldraw[black] (17,12) circle (3pt);
			\filldraw[black] (17,10) circle (3pt);
			\filldraw[black] (16,11) circle (3pt);
			\filldraw[black] (18,11) circle (3pt);
			
			\draw[brown,line width=4pt, dotted](19.3,15.3) -- (21.7,17.7);
			
			\draw[violet, thick, dashed] (19,21) ellipse (1 and 1);
			\filldraw[black] (19,20) circle (3pt);
			\filldraw[black] (19,22) circle (3pt);
			\filldraw[black] (18,21) circle (3pt);
			\filldraw[black] (20,21) circle (3pt);
			
			\draw[violet, thick, dashed] (25,19) ellipse (1 and 1);
			\filldraw[black] (25,18) circle (3pt);
			\filldraw[black] (25,20) circle (3pt);
			\filldraw[black] (24,19) circle (3pt);
			\filldraw[black] (26,19) circle (3pt);

	\end{tikzpicture}}
	
	\caption{Illustration of the point-set $X$ with the explanation from a decision tree of depth $k-1$.}
	\label{fig-kc}
\end{figure}
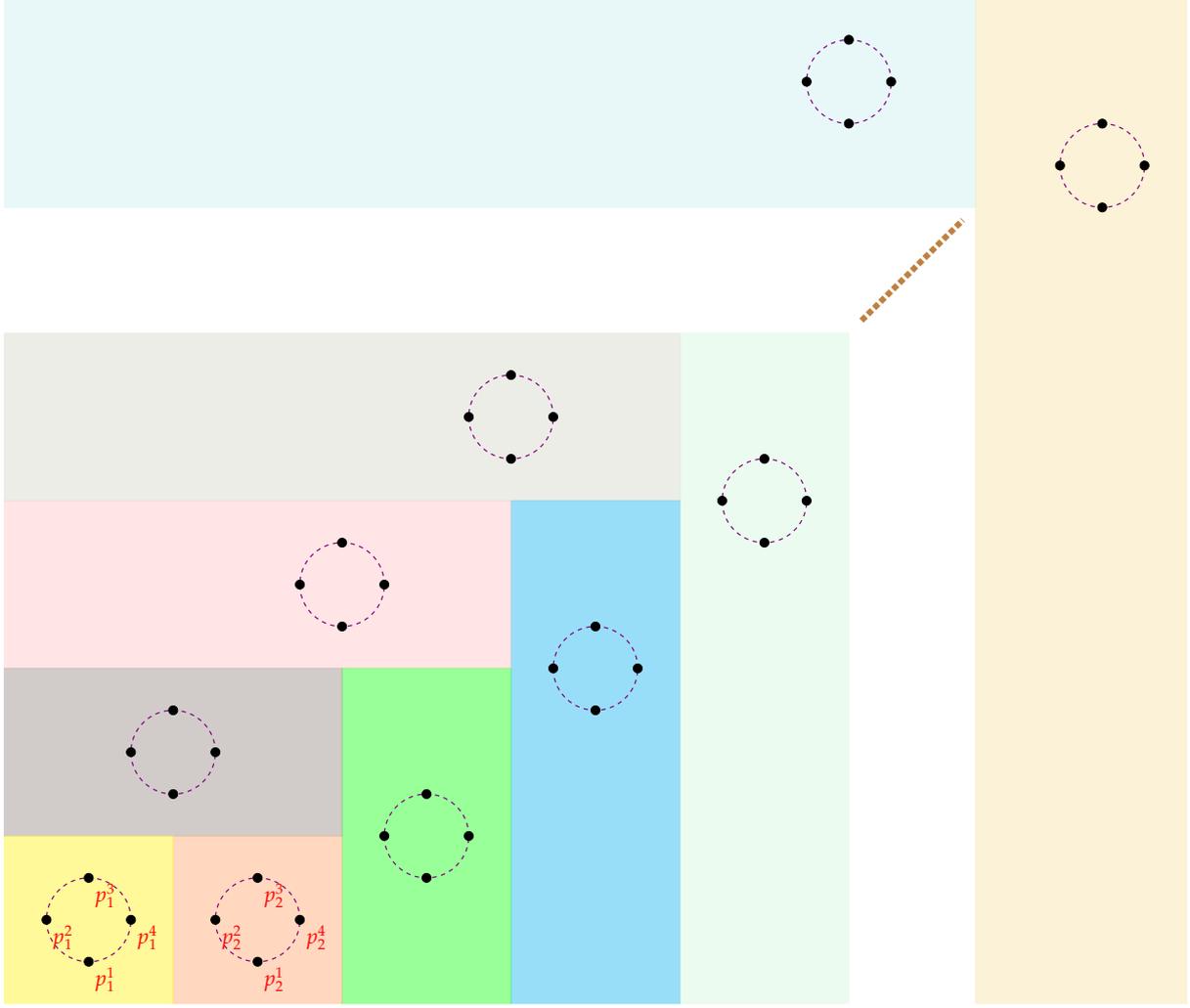

\subsection{Proof of Theorem \ref{thm:center}}

We start with the correctness proof in \Cref{lem:center-correctness}. The price of depth reduction can be concluded by \Cref{lem:center-depth-reduction} with \Cref{lem:center-price}.
\begin{lemma}
	\label{lem:center-correctness}
	There exists a decision tree of depth $k-1$ which outputs the optimal clustering for point-set $X$ w.r.t. the $k$-center objective.
\end{lemma}
\begin{proof}
	Consider the following sequence of  $k-1$ threshold cuts. For $i=k,k-1,\ldots,2$:
	\begin{itemize}
		\item If $i$ is even, create a cut along $x=i-3/2$.
		\item If $i$ is odd, create a cut along $y=i-3/2$.
	\end{itemize} 
	We can see that the first cut partitions the first cluster away from the rest of the point set entirely. This applies to all the other cuts as well. Hence, the above sequence transforms to a decision tree of depth $k-1$ with a leaf in every level and outputs an explainable clustering identical to the optimal clustering.
\end{proof}
\begin{lemma}
	\label{lem:center-depth-reduction}
	Every decision tree of depth less than $k-1$ outputs a clustering which misclassifies (w.r.t. the optimal clustering $C$) at least one point.
\end{lemma}
\begin{proof}
	The proof of this lemma is similar to the proof of Lemma \ref{lem:imp-misclassify}. We prove the base case for $3$ clusters and show that the only possible cuts in the first level of the decision tree are those which separate out the cluster $C_k$ from the remaining points. We can then use induction on $k$ to show that any decision tree that does not misclassify any point has depth $k-1$
\end{proof}
\begin{lemma}
	\label{lem:center-price}
	Consider any other choice of explainable clustering $\Tilde{C} \neq C$. The $k$-center objective of $\Tilde{C}$ is at least $1$.
\end{lemma}
\begin{proof}
	Consider any clustering $\Tilde{C} \neq C$ with $k$-center cost less than $1$. Firstly, for all $i \geq 5$, the distance between points in the cluster $C_i$ and points in any other cluster is at least $2$. Hence, if any cluster contains a point from a cluster $C_i$, where $i \geq 5$ and a point from a cluster $C_j$, where $j \neq i$, the radius of the circle containing both these points is at least $2$, means that the center of the cluster these two points are assigned to must be distance at least $1$ from one of them, which implies that the $k$-center cost of this clustering must be at least $1$. Hence, the clustering $\Tilde{C}$ must either split some of these clusters into multiple clusters, or leave them intact. 
	
	Let us first consider the case that $\Tilde{C}$ splits at least one of these clusters. As there are only $k$ clusters in total, that means that the points in the first $4$ clusters have to be divided into at most $3$ clusters. However, consider the points $p_1^1, p_2^1,p_3^1, p_4^1$. The pairwise distances between each of these points is at least $2$, and hence, for the $k$-center cost to be less than $1$, they all have to be in different clusters which is not possible as there are only $3$ clusters to divide them into. 
	
	We now consider the case that $\Tilde{C}$ leaves all the clusters $C_i$ for $i \geq 5$ intact. In this case, it has to induce a different clustering among the points in the first $4$ clusters. As in the previous case, the points $p_1^1, p_2^1,p_3^1, p_4^1$ have to be in separate clusters because their pariwise distances are larger than $2$.
	
	Consider the point $p_4^2, p_4^3$ and $p_4^4$. They are at a distance more than $2$ from every one of these point apart from $p_4^1$, and hence have to be assigned to the cluster containing $p_4^1$. Note that the point $p_4^4$ is at a distance greater than $2$ from every point outside $C_4$. This implies that $\Tilde{C}$ must preserve the cluster $C_4$.
	
	Next we prove that $\Tilde{C}$ preserves the cluster $C_3$. The point $p_3^3$ has to be assigned to the cluster containing $p_3^1$, and it is at a distance of at least $2$ from every point outside $C_3$. This also implies that the cluster containing $p_3^1$ can only contain points from $C_3$. Note that the points $p_3^2$ and $p_3^4$ are at distances greater than $2$ from $p_1^1$ and $p_2^1$ which implies that $\Tilde{C}$ must preserve the cluster $C_3$. 
	
	We can now show that the points $p_1^2$ and $p_1^3$ cannot be assigned to the cluster containing $p_2^1$ and hence must be assigned to the cluster containing $p_1^1$.  This implies that $p_2^4$, $p_2^3$ and $p_2^2$ must be assigned to the cluster containing $p_2^1$. Finally this also implies that $p_1^4$ is assigned to the cluster contiaining $p_1^1$. Hence, $\Tilde{C}$ must preserve the clusters $C_1$ and $C_2$ as well.
\end{proof}

\section{Discussion and Open Problems}\label{sec:conclusion}

In this paper, we introduced the notion of price of depth reduction, and showed that even in the plane, there is a dataset for which it is unbounded for the $k$-means and $k$-median objectives (and a weaker result for the $k$-center objective). 

\subsection{Other Metric Spaces}\label{sec:metric}
\label{subsec:metric}
It is easy to see that the lower bound in Section 3 of \cite{moshkovitz2020explainable} for high dimensions   extends to other $\ell_p$-metrics. Moreover, even in $\mathbb{R}^2$, it is possible to extend Theorem~\ref{thm:2d-impossible} to other 
$\ell_p$-metrics by choosing the sequences $\{w_i\}_{i=1}^k$ and $\{h_i\}_{i=1}^k$ to be suited  to the $\ell_p$-metric space (i.e., the value of $p$). 

The above remark is particularly relevant as Dasgupta et al.\ \cite{moshkovitz2020explainable}   introduced explainable clustering for the $k$-median objective in the $\ell_1$-metric, as it is easier to analyze the $k$-median objective in that metric space. Explainable clustering for the $k$-median objective in the $\ell_2$-metric was only studied in the later work of \cite{MakarychevS21}. However, as stated above, our lower bounds on the price of depth reduction can be extended to the $k$-median objective in the $\ell_1$-metric as well. 

\subsection{Implications of Our Lower Bound Construction}
One may wonder if our lower bound construction is too pathological. To address this concern, we now informally argue that our impossibility result gives a characterization for a non-negligible class of datasets w.r.t.\ its optimal decision tree for explainable clustering.  

Call a rooted decision tree ``totally skewed'' if it's depth is 1 less than the number of leaves in the tree. We note that in a typical decision tree of large depth (say of depth linear in the number of leaves), there is a large induced connected subgraph which is a totally skewed subtree. 

We also note that if the optimal explanation for a dataset is a totally skewed tree then the dataset is "isomorphic" to our seemingly pathological dataset given in \Cref{sec:plane}. To see this, consider all the cuts except the first cut. If the last (closest to the root) cluster didn’t intersect the second cut, we can use all the cuts except the first and add an additional cut somewhere else in the subtree to separate the last cluster, thus getting a smaller depth tree. So, the second cut intersects the last cluster. This means that the last cluster has to block the gap between some previous clusters. We can repeat this argument to show that the clusters are added in layers blocking the space between previous clusters.



Putting these two observations together, if we only have access to the optimal explanation of a dataset, then we know that there is  a set of coordinates (the ones appearing in the nodes of the subtree which is totally skewed), 
under whose projection, there is a subset of the input points which looks similar (after scaling and rotation) to the point-set in Section~\ref{sec:plane}.

\subsection{Open Problems}
Finally, we dedicate the rest of the section to highlight a couple of open problems. 

\paragraph{Open Problem 1.} Our  dataset construction for the price of depth reduction for the $k$-center objective in the plane, only gave a lower bound of (slightly less than) two. We believe that it might not be possible to improve our lower bound any further and leave it as  an important open problem to find an  algorithm (if it exists) that provides a shallow decision tree explanation for all datasets in the plane and whose price of depth reduction is upper bounded by 2 (or at least by some constant independent of $k$).

\paragraph{Open Problem 2.} We think that an intriguing direction for future research is to explore bicriteria approximation of price of depth reduction. 
Elaborating, we wonder  if it is possible to obtain a tradeoff between the number of leaves (clusters) in the decision tree and the price of depth reduction. This suggested research direction is a natural extension of  \cite{makarychev2022explainable}, where the authors showed a near optimal tradeoff between the number of leaves in the decision tree and the price of explanation.

Additionally, we think it might be possible that the algorithm given in \cite{esfandiari2022almost} may yield a low depth decision tree (with close to the optimal clustering cost), if the dataset is in \emph{fixed dimensions} and admits a low depth decision tree (with optimal clustering cost). We leave this exploration for future research.

 \subsubsection*{Acknowledgements}
 We thank Vincent Cohen-Addad for some preliminary discussions. 
\bibliographystyle{alpha}
\bibliography{reference}

\end{document}